\documentclass[conference,a4paper,onecolumn]{IEEEtran}
\usepackage{url}
\usepackage{times}
\usepackage{graphicx,cite,calc}
\usepackage{caption}
\usepackage{subcaption}

\usepackage{pgfplots,tikz}
\usetikzlibrary{patterns}
\usetikzlibrary{fit,calc,positioning,decorations.pathreplacing,matrix,decorations.markings,shapes,spy}
\usetikzlibrary{arrows, positioning, calc,matrix,shadows,fadings,shapes.arrows,patterns,snakes}
\tikzstyle{line}=[draw]
\tikzstyle{arrow}=[draw, -latex] 
\usetikzlibrary{3d}

\usepackage{amsfonts,amssymb}
\usepackage{amsmath,algorithm,algorithmic,amsthm}

\newtheorem{theorem}{Theorem}

\newtheorem{lemma}[theorem]{Lemma}

\newcommand{\beq}{\begin{equation}}
\newcommand{\eeq}{\end{equation}}
\newcommand{\beqa}{\begin{eqnarray}}
\newcommand{\eeqa}{\end{eqnarray}}
\newcommand{\bit}{\begin{itemize}}
\newcommand{\eit}{\end{itemize}}
\newcommand{\ben}{\begin{enumerate}}
\newcommand{\een}{\end{enumerate}}
\newcommand{\mc}{\mathcal}
\newcommand{\mb}{\mathbb}
\newcommand{\bed}{\begin{displaymath}}
\newcommand{\eed}{\end{displaymath}}

\pgfkeys{/tikz/.cd,
  arrow color/.initial=black!80!white,
  arrow color/.get=\arrowcol,
  arrow color/.store in=\arrowcol,
  items distance/.initial=2.5cm,
  items distance/.get=\itemdistance,
  items distance/.store in=\itemdistance,
  border color/.initial=black!80!white,
  border color/.get=\bordercol,
  border color/.store in=\bordercol,
  fill color/.initial=green!40!black!20,
  fill color/.get=\fillcol,
  fill color/.store in=\fillcol,
  brace color/.initial=black,
  brace color/.get=\bracecol,
  brace color/.store in=\bracecol,
  brace distance/.initial=5pt,
  brace distance/.get=\bracedistance,
  brace distance/.store in=\bracedistance,
}

\tikzset{my arrow/.style={
    single arrow, draw, minimum height=1.75cm,
    minimum width=2.5cm,
    single arrow head extend=0.1cm
  },
brace/.style={
    decoration={brace,mirror,raise=\bracedistance,amplitude=0.75em},
    decorate,
    draw=\bracecol,
    very thick,    
  },
brace_nonmirror/.style={
    decoration={brace,raise=\bracedistance,amplitude=0.75em},
    decorate,
    draw=\bracecol,
    very thick,    
  }
}

\tikzset{xcenter around/.style 2 args={execute at end picture={%
  \useasboundingbox let \p0 = (current bounding box.south west), \p1 = (current bounding box.north east),
                        \p2 = (#1), \p3 = (#2)
                    in
        ({min(\x2 + \x3 - \x1,\x0)},\y0) rectangle ({max(\x3 + \x2 - \x0,\x1)},\y1);
}}}
\begin{document}
\sloppy
\title{Coupled Neural Associative Memories}
\author{
  \IEEEauthorblockN{Amin Karbasi, Amir Hesam Salavati, and Amin Shokrollahi}
  \IEEEauthorblockA{School of Computer and Communication Sciences\\ Ecole Polyetechnique Federale de Lausanne (EPFL)\\ Switzerland} 

}

\maketitle

\begin{abstract}
We propose a novel architecture to design a neural associative memory that is capable of learning a large number of patterns and recalling them later in presence of noise. It is based on dividing the neurons into local clusters and parallel plains, very  similar to the architecture of the visual cortex of macaque brain. The common features of our proposed architecture with those of spatially-coupled codes  enable us to show that the performance of such networks in eliminating noise is drastically better than the previous approaches while maintaining the ability of learning an exponentially large number of patterns. Previous work either failed in providing good performance during the recall phase or in offering large pattern retrieval (storage) capacities. We also present computational experiments that lend additional support to the theoretical analysis. 
\end{abstract}

\section{Introduction}
The ability of the brain to memorize large quantities of data and later recalling them from partially available information is truly staggering. 
While relying on iterative operations of simple (and sometimes faulty) neurons, our brain is capable of retrieving the correct "memory" with high degrees of reliability even when the cues are limited or inaccurate.

Not surprisingly, designing artificial neural networks capable of accomplishing this task, called \emph{associative memory}, has been a major point of interest in the neuroscience community for the past three decades. This problem, in its core, is very similar to the reliable information transmission faced in communication systems where the goal is to find mechanisms to efficiently encode and decode a set of transmitted patterns over a noisy channel. More interestingly, the novel techniques employed to design good codes are extremely similar to those used in designing and analyzing neural networks. In both cases, graphical models, iterative algorithms, and message passing play central roles.

Despite these similarities in the objectives and techniques, we witness a huge gap in terms of the efficiency achieved by them. More specifically, by using modern coding techniques, we are capable of reliably transmitting $2^{rn}$ binary vectors of length $n$ over a noisy channel ($0<r<1$). This is achieved by intelligently introducing redundancy among the transmitted messages, which is later used to recover the correct pattern from the received noisy version. In contrast, until recently, artificial neural associative memories were only capable of memorizing $O(n)$ binary patterns of length $n$ (see, \cite{hopfield}, \cite{venkatesh}, \cite{Jankowski}, \cite{Muezzinoglu1}). 

Part of the reasons for this gap goes back to the assumption held in the mainstream work on artificial associative memories which requires the network to memorize \emph{any} set of randomly chosen binary patterns. While it gives the network a certain degree of versatility, it severely hinders the efficiency.

To achieve an exponential scaling in the storage capacity of neural networks Kumar et al. \cite{KSS} suggested a different viewpoint in which the network is no longer required to memorize \emph{any} set of random patterns but only those that have some common \emph{structure}, namely, patterns all belong to a subspace with dimension $k<n$. Karbasi et al. \cite{KSS_ICML2013} extended this model to "modular" neural architectures and introduced a suitable online learning algorithm. They showed that the modular structure improves the noise tolerance properties significantly.

In this work, we extend the model of \cite{KSS_ICML2013} further by linking the modular structures to obtain a "coupled" neural architecture. Interestingly, this model looks very similar to some models for processing visual signals in the macaque brain \cite{acm_brain}. We then make use of the recent developments in the analysis of spatially-coupled codes by \cite{pfister_istc} and \cite{kudekar} to derive analytical bounds on the performance of the proposed method. Finally, using simulations we show that the proposed method achieves much better performance measures compared to previous work in eliminating noise during the recall phase.

\section{Related Work}\label{sec:related}
Arguably, one of the most influential models for  neural associative memories was introduced by Hopfield \cite{hopfield}. 
%
%
A "Hopfield network" is a complete graph of $n$ neurons that memorizes a subset of \emph{randomly} chosen binary patterns of length $n$. It is known that the pattern retrieval capacity (i.e., maximum number of memorized patterns) of Hopfield networks is $\mc{C} = (n/2\log(n))$ \cite{mceliece}.

There have been many attempts to increase the pattern retrieval capacity of such networks by introducing offline learning schemes (in contrast to online schemes) \cite{venkatesh} or multi-state neurons (instead of binary) \cite{Muezzinoglu1}, all of which resulted in memorizing at most $O(n)$ patterns. 

By dividing the neural architecture into smaller \emph{disjoint} blocks, Venkatesh \cite{venkatesh_exponential} increased the capacity to $\Theta \left(b^{n/b} \right)$ (for \emph{random} patterns), where $b = \omega( \ln n)$ is the size of blocks. This is a huge improvement but comes at the price of limited \emph{worst-case} noise tolerance capabilities. Specifically, due to the non-overlapping nature of the clusters (blocks), the error correction is limited by the performance of individual clusters as there is no inter-cluster communication. With \emph{overlapping} clusters, one could hope for achieving better error correction, which is the reason we consider such structures in this paper. 

More recently, a new perspective has been proposed with the aim of memorizing only those patterns that posses some degree of redundancy. In this framework, a tradeoff is being made between versatility (i.e., the capability of the network to memorize any set of random patterns) and the pattern retrieval capacity. Pioneering this frontier, Gripon and Berrou \cite{gripon_sparse} proposed a method based on neural clicks which increases the pattern retrieval capacity potentially to $O(n^2/\log(n))$ with a low complexity algorithm in the recall phase. The proposed approach is based on memorizing a set of patterns mapped from randomly chosen binary vectors of length $k = O(\log(n))$ to the $n$-dimensional space. 
Along the same lines, by considering patterns that lie in a subspace of dimension $k<n$, Kumar et al. were able to show an exponential scaling in the pattern retrieval capacity, i.e., $\mc{C} = O(a^n)$, with some $a>1$. This model was later extended to modular patterns, i.e., those in which patterns are divided into sub-patterns where each sub-pattern come from a subspace \cite{KSS_ICML2013}. The authors provided a simple iterative learning algorithm that demonstrates a better performance in the recall phase as compared to \cite{KSS}.

In this paper, we follow the same line of thought by linking several instances of the model proposed in \cite{KSS_ICML2013} in roder to have a "coupled" structure. More specifically, the proposed model is based on \emph{overlapping} local clusters, arranged in parallel planes,   with neighboring neurons. At the same time, we enforce sparse connections between various clusters in different planes. The aim is to memorize only those patterns for which local sub-patterns in the domain of each cluster show a certain degree of redundancy.
On the one hand, the obtained model looks similar to  neural modules in the visual cortex of the macaque brain \cite{acm_brain}. And on the other hand, it is  similar to spatially-coupled codes on graphs \cite{kudekar}. Specifically, our suggested model is closely related to the spatially-coupled Generalized LDPC code (GLDPC) with Hard Decision Decoding (HDD) proposed in \cite{pfister_hdd}. This similarity helps us borrow analytical tools developed for analyzing such codes \cite{pfister_istc} and investigate the performance of our proposed neural error correcting algorithm. 

The proposed approach enjoys the simplicity of message passing operations performed by neurons as compared to the more complex iterative belief propagation decoding procedure of spatially coupled codes \cite{kudekar}. This simplicity may lead to an inferior performance but already allows us to outperform the prior error resilient methods suggested for neural associative memories in the literature. 

%

\section{Problem Setting and Notations}\label{sec:model}
In this paper, we work with non-binary neural networks, where the state of each neurons is a bounded non-negative integer (which can be thought of as the short-term firing rate of neurons in a real neural network). Like other neural networks, neurons could only perform simple operations, i.e. \textit{linear summation} and \textit{non-linear thresholding}. More specifically, a neuron $x$ updates its state based on the states of its neighbors $\{s_i\}_{i=1}^{n}$ as follows:
\begin{enumerate}
\item It computes the weighted sum
$ h = \sum_{i=1}^{n} w_i s_i,$ 
where $w_i$ denotes the weight of the input link from $s_i$.
\item It updates its state as $x = f(h),$
where $f: \mb{R} \rightarrow \mc{S}$ is a possibly non-linear function.
\end{enumerate}

Let $\mc{X}$ denote a dataset of $\mc{C}$ patterns of length $n$ where the patterns are integer-valued with entries in $\mc{S} = \{0,\dots, S-1\}$. A natural way of interpreting this model is to consider the entries as the short-term firing rate of $n$ neuron. In this paper we are interested in designing a neural network that is capable of memorizing these patterns in such a way that later, and in response to noisy queries, the correct pattern will be returned. To this end, we break the patterns into smaller pieces/sub-patterns and learn the resulting sub-patterns separately (and in parallel). Furthermore, as our objective is to memorize those patterns that are highly correlated, we only consider a dataset in which the sub-patterns belong to a subspace (or have negligible minor components).

More specifically, and to formalize the problem in a way which is similar to the literature on spatially coupled codes, we divide each pattern into $L$ sub-patterns of the same size and refer to them as \textit{planes}. Within each plane, we further divide the patterns into $D$ \emph{overlapping} clusters, i.e., an entry in a pattern can lie in the domain of multiple clusters. We also assume that each element in plane $\ell$ is connected to at least one cluster in planes $\ell-\Omega,\dots,\ell+\Omega$ (except at the boundaries). Therefore, each entry in a pattern is connected to $2\Omega+1$ planes, on average. 

An alternative way of understanding the model is to consider  2D datasets, i.e., images. In this regard, each row of the image corresponds to a plane and clusters are the overlapping "receptive fields" which cover an area composed of neighboring pixels in different rows (planes). This is in fact very similar to the configuration of the receptive fields in human visual cortex \cite{dyan}. Our assumptions on strong correlations then translates into assuming strong linear dependencies within each receptive field for all patterns in the dataset.

\textbf{Noise model:} Throughout this paper, we consider an additive noise model. More specifically, the noise  is an integer-valued vector of size $n$ and for simplicity we assume that its entries are $\{-1,0,+1\}$, where a $-1$ (resp. $+1$) corresponds to a neuron skipping a spike (resp. fire one more spike than expected).\footnote{Other noise models, such as real-valued noise, can be considered as well. However, the thresholding function $f:\mb{R}\rightarrow \mc{S}$ will eventually lead to integer-valued "equivalent" noise in our system.} The noise probability is denoted by $p_e$ and each entry of the noise vector is $+1$ or $-1$ with probability $p_e/2$. \footnote{Our algorithm can also deal  with erasures. Note that an erasure at node $x_i$ corresponds to an integer noise with the negative value of $x_i$. So once we have established the performance of our algorithm for integer-valued noise, it would be straightforward to extend the algorithm to erasure noise models.} 

\textbf{Pattern Retrieval Capacity:} This is the maximum number of patterns that can be memorized by a network while still being able to return reliable responses in the recall phase.

\section{Learning Phase} 
To "memorize" the patterns, we learn a set of vectors that are orthogonal to the sub-patterns in each cluster, using the algorithm proposed in \cite{KSS_ICML2013}. The output of the learning algorithm is an $m_{\ell,d} \times n_{\ell,d}$ matrix $W^{(\ell,d)}$ for cluster $d$ in plane $\ell$. The rows of this matrix correspond to the dual vectors and the columns correspond to the corresponding entries in the patterns. Therefore, by letting $\textbf{x}^{(\ell,d)}$ denote the sub-pattern corresponding to the domain of cluster $d$ of plane $\ell$, we have
\begin{equation}\label{eq:orthogonal}
W^{(\ell,d)}\cdot \textbf{x}^{(\ell,d)} = \textbf{0}.
\end{equation}

These matrices (i.e., $W^{(\ell,d)}$) form the connectivity matrices of our neural graph, in which we can consider each cluster as a bipartite graph composed of \emph{pattern} and \emph{constraint} neurons. The left panel of Figure~\ref{fig:coupled_model} illustrates the model, in which the circles and rectangles correspond to pattern and constraint neurons, respectively. The details of the first plane are magnified in the right panel of Figure~\ref{fig:coupled_model}. 

\begin{figure}[t]%
\centering
\input{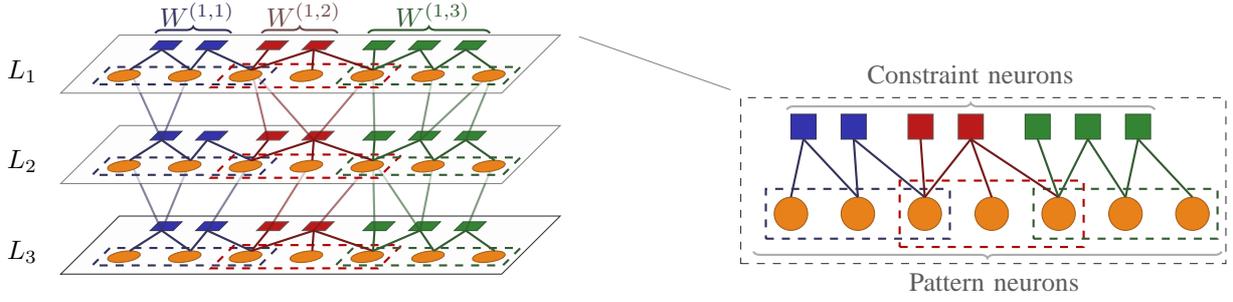}
\caption{A coupled neural associative memory.}
\label{fig:coupled_model}
\end{figure}

Cluster $d$ in plane $\ell$ contains $m_{\ell,d}$ constraint neurons and is connected to $n_{\ell,d}$ pattern neurons. The constraint neurons do not have any overlaps (i.e. each one belongs only to one cluster) whereas the pattern neurons can have connections to multiple clusters and planes. To ensure good error correction capabilities we aim to keep the neural graph sparse (this model shows significant  similarity to some neural architectures in the macaque brain \cite{acm_brain}). 

We also consider the overall connectivity graph of plane $\ell$, denoted by $\widetilde{W}^{(\ell)}$, in which the constraint nodes in each cluster are compressed into one \emph{super node}. Any pattern node that is connected to a given cluster is connected with an (unweighted) edge to the corresponding super node. Figure~\ref{super_graph_model} illustrates this graph for plane $1$ in Figure \ref{fig:coupled_model}.

\begin{figure}[b]
\centering
\begin{tikzpicture}[scale=.36]

\draw[draw=black!60!gray!100,fill=gray!2] (-3,-3) rectangle (26,7);

\draw[draw=black!80!blue!80,fill=black!40!blue!80,rotate around={45:(4,4.5)}] (4,4.5) rectangle (5.5,6);
\draw[draw=black!80!red!70,fill=black!30!red!90,rotate around={45:(12,4.5)}] (12,4.5) rectangle (13.5,6) ;
\draw[draw=black!80!green!80,fill=black!60!green!80,rotate around={45:(20,4.5)}] (20,4.5) rectangle (21.5,6) ;

\draw[draw=black!40!orange!90,fill=black!10!orange!90] (0,0) circle (1cm);
\draw[draw=black!40!orange!90,fill=black!10!orange!90] (4,0) circle (1cm);
\draw[draw=black!40!orange!90,fill=black!10!orange!90] (8,0) circle (1cm);
\draw[draw=black!40!orange!90,fill=black!10!orange!90] (12,0) circle (1cm);
\draw[draw=black!40!orange!90,fill=black!10!orange!90] (16,0) circle (1cm);
\draw[draw=black!40!orange!90,fill=black!10!orange!90] (20,0) circle (1cm);
\draw[draw=black!40!orange!90,fill=black!10!orange!90] (24,0) circle (1cm);

\draw[draw=black!80!blue!80,dashed,thick] (-1.5,-1.5) rectangle (9.5,1.5);

\draw[draw=black!30!red!100,dashed,thick] (6.5,-2) rectangle (17.5,2);

\draw[draw=black!80!green!80,dashed,thick] (14.5,-1.5) rectangle (25.5,1.5);

\draw[thick,black!80!blue!80] (0,1) --(4,4.5);
\draw[thick,black!80!blue!80] (4,1) --(4,4.5);
\draw[thick,black!80!blue!80] (8,1) --(4,4.5);

\draw[thick,black!60!red!90] (8,1) -- (12,4.5);
\draw[thick,black!60!red!90] (12,1) -- (12,4.5);
\draw[thick,black!60!red!90] (16,1) -- (12,4.5);

\draw[thick,black!80!green!80] (16,1) --(20,4.5);
\draw[thick,black!80!green!80] (20,1) --(20,4.5);
\draw[thick,black!80!green!80] (24,1) --(20,4.5);

\end{tikzpicture}
\caption{A connectivity graph with neural planes and super nodes. It corresponds to plane $1$ of Fig.~\ref{fig:coupled_model}.}
\label{super_graph_model}
\end{figure}
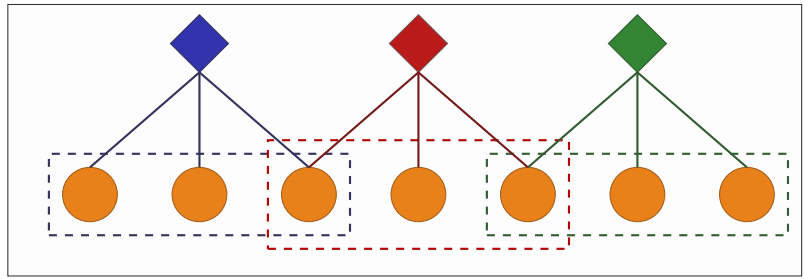

\section{Recall Phase}\label{sec:recall}
The main goal of our architecture is to retrieve correct memorized patterns in response to noisy queries. At this point, the neural graph is learned (fixed) and we are looking for a simple iterative algorithm to eliminate noise from  queries. The proposed recall algorithm in this paper is the extension of the one in \cite{KSS_ICML2013} to the coupled neural networks. For the sake of completeness, we briefly discuss the details of the approach suggested in \cite{KSS_ICML2013} and explain the extension subsequently. 

The recall method in \cite{KSS_ICML2013} is composed of two types of separate algorithms: \emph{local} (or intra-cluster) and \emph{global} (or inter-cluster). The local algorithm tries to correct errors within each cluster by the means of simple message-passings. It relies on 1) pattern neurons transmitting their state to the constraint neurons and then on 2) constraint neurons checking if the constraints are met (i.e. the values transmitted by the pattern nodes to the constraint nodes should sum to zero). If not, the constraint neurons send a message telling the direction of the violation (i.e. if the input sum is less or greater than zero). The pattern neurons then updates their state according to the received feedback from their neighboring constraint neurons on a majority voting basis. The process is summarized in Algorithm~\ref{algo:correction}.

The overall error correction properties of Algorithm~\ref{algo:correction} is fairly limited. In fact, it can be shown that in a given cluster, the algorithm could correct a single input error (i.e only one pattern neurons deviating from its correct state) with probability  $1- (\bar{d}/m)^{d_{\min}}$, where $\bar{d}$ and $d_{\min}$ are the average and minimum degree of the pattern nodes. For more than one input error, the algorithm can easily get stuck. To overcome this drawback, Karbasi et al. \cite{KSS_ICML2013} proposed a sequential procedure by applying Algorithm~\ref{algo:correction} in a Round Robbin fashion to each cluster. If the errors were eliminated, the pattern nodes in the cluster keep their new values, and revert back to their original states, otherwise. This scheduling technique is in esprit similar to the Peeling Algorithm widely used in LDPC codes \cite{erasure}. Correcting the error in the clusters with a single error can potentially help the neighboring clusters. 

%
%

Inspired by this boost in the performance, we can stretch the error correction capabilities even further by coupling several neural "planes" with many clusters together, as mentioned earlier. We need to modify the global error correcting algorithm in such a way that it first acts upon the clusters of a given plane in each round before moving to the next plane. The whole process is repeated few times until all errors are corrected or a threshold on the number of iterations is reached ($t_{\max}$).  Algorithm~\ref{algo:coupled} summarizes our approach.

\begin{algorithm}[t]
\caption{Error Correction Within Cluster \cite{SK_ISIT2012}}
\label{algo:correction}
\begin{algorithmic}[1]
\REQUIRE{Connectivity matrix $W^{(\ell,d)}$, threshold $\varphi$, iteration $t_{\max}$.}
\ENSURE{Correct memorized sub-pattern $\textbf{x}^{(\ell,d)}$.}
\FOR{$t = 1 \to t_{\max}$} 
\STATE \textit{Forward iteration:} Calculate the weighted input sum $ h_i = \sum_{j=1}^n W^{(\ell,d)}_{ij} x^{(\ell,d)}_j,$ for each neuron $y^{(\ell,d)}_i$ and set:
\[ y^{(\ell,d)}_i = \left\{
\begin{array}{cc}
1, & h_i < 0\\
0, & h_i = 0\\
-1, & \hbox{otherwise}
\end{array} \right. 
\]
\STATE \textit{Backward iteration:} Each neuron $x^{(\ell,d)}_j$ computes
\[ g^{(\ell,d)}_j = \frac{\sum_{i=1}^{m_{\ell,d}} W^{(\ell,d)}_{ij} y^{(\ell,d)}_i }{\sum_{i = 1}^{m_{\ell,d}}|W^{(\ell,d)}_{ij}|}. \]
\STATE Update the state of each pattern neuron $j$ according to $x^{(\ell,d)}_j = x^{(\ell,d)}_j + \hbox{sgn}(g^{(\ell,d)}_j)$
only if $|g^{(\ell,d)}_j| > \varphi$.
\ENDFOR
\end{algorithmic}
\end{algorithm}
\begin{algorithm}[t]
\caption{Error Correction of the Coupled Network}
\label{algo:coupled}
\begin{algorithmic}[1]
\REQUIRE{Connectivity matrix ($W^{(\ell,d)},\forall \ell,\forall d$), iteration $t_{\max}$}
\ENSURE{Correct memorized pattern $\textbf{x} =[x_1,x_2,\dots,x_{n}]$}
\FOR{$t = 1 \to t_{\max}$} 
\FOR{$\ell = 1 \to L$} 
\FOR{$d = 1 \to D$}
\STATE Apply Algorithm~\ref{algo:correction} to cluster $d$ of neural plane $\ell$.
\STATE Update the value of pattern nodes $\textbf{x}^{(\ell,d)}$ only if all the constraints in the clustered are satisfied. 
\ENDFOR
\ENDFOR 
\ENDFOR
\end{algorithmic}
\end{algorithm}
\vspace{0.15cm}


\section{Performance Analysis}
We consider two variants of the above error correction algorithm. In the first one, called \emph{constrained} coupled neural error correction, we provide the network with some side information during the recall phase. This is equivalent to "freezing" a few of the pattern neurons to known and correct states, similar to spatially-coupled codes \cite{pfister_istc}, \cite{kudekar}. In the case of neural associative memory, the side information can come from the context. For instance, when trying to correct the error in the sentence "The \emph{\underline{c}}at flies", we can use the side information (flying) to guess the correct answer among multiple choices. Without this side information, we cannot tell if the correct answer corresponds to \emph{bat} or \emph{rat}.\footnote{The same situation also happens in dealing with erasures, i.e. when trying to fill in the blank in the sentence "The \emph{$\_$at} flies".}

In the other variant, called \emph{unconstrained} coupled neural error correction, we perform the error correction without providing any side information. This is similar to many standard recall algorithms in neural networks. In fact, the unconstrained model can be thought of as a very large convolutional network similar to the model proposed in \cite{KSS_ICML2013}. Thus, the unconstrained model serves as a benchmark to evaluate the performance of the proposed coupled model in this paper. 

Let $z^{(\ell)}(t)$ denote the \emph{average} probability of error for pattern nodes across neural plane $\ell$ in iteration $t$. Thus, a super constraint node in plane $\ell$ receives noisy messages from its neighbors with an average probability $\bar{z}^{(\ell)}$:
$$\bar{z}^{(\ell)} = \frac{1}{2\Omega+1}\sum_{j=-\Omega}^{\Omega} z^{(\ell-j)} ~ \text{s.t.}~ z^{(l)} = 0,~\forall l\notin \{1,\dots,L\}.$$

Our goal is to derive a recursion for $z^{(\ell)}(t+1)$ in terms of $z^{(\ell)}(t)$ and $\bar{z}^{(\ell)}(t)$. To this end, in the graph $\widetilde{W}^{(\ell)}$ let $\lambda_i^{(\ell)}$ and $\rho^{(\ell)}_j$ be the fraction of edges (normalized by the total number of edges in graph $\widetilde{W}^{(\ell)}$) connected to pattern and super nodes with degree $i$ and $j$, respectively. We define the degree distribution polynomials in plane $\ell$ from an \textit{edge perspective} as $\lambda^{(\ell)}(x) = \sum_i \lambda^{(\ell)}_i x^{i}$ and $\rho^{(\ell)}(x) = \sum_j \rho^{(\ell)}_i x^{j-1}$. 

\begin{lemma}\label{lem:recursive}
Let us define $g(z)=1-\rho(1-z)-\sum_{i=1}^{e-1} \frac{z^i}{i!}\frac{d^i\rho(1-z)}{dz^i}$ and $f(z;p_e) = p_e \lambda(z)$, where $e$ is the number of errors each cluster can correct. Then,
\beq \label{z_ell}
z^{(\ell)}(t+1) = f\left(\frac{1}{2\Omega + 1}\sum_{i=-\Omega}^{\Omega}g(\bar{z}^{(\ell-i)}(t));p_e\right).
\eeq
\end{lemma}
\begin{proof}
Without loss of generality, we prove the lemma for the case that each cluster could correct at least two errors with high probability, i.e. $e=2$. Extending the proof to $e>2$ would be straightforward.

Let $z^{(\ell)}(t)$ denote the \emph{average} probability of error for pattern nodes across neural plane $\ell$ and in iteration $t$. Furthermore, let $\pi^{(\ell)}(t)$ be the \emph{average} probability of a \emph{super constraint node} in plane $\ell$ sending an erroneous message to its neighbors. We will derive recursive expressions for $z^{(\ell)}(t)$ and $\pi^{(\ell)}(t)$. 

A super constraint node in plane $\ell$ receives noisy messages from its neighbors with an average probability of $\bar{z}^{(\ell)}$, where 
$$\bar{z}^{(\ell)} = \frac{1}{2\Omega+1}\sum_{j=-\Omega}^{\Omega} z^{(\ell-j)}$$
with $z^{(i)} = 0$ for $i \leq 0$ and $i > L$. 

Let $\pi_i^{(\ell)}$ denote the the probability that a super constraint node with degree $i$ in plane $\ell$ sends an erroneous message to one of its neighboring noisy pattern nodes. Then, knowing that each super constraint node (cluster) is capable of correcting at least $e=2$ errors, $\pi_i^{(\ell)}$ is equal to the probability of receiving two or more noisy messages from \emph{other} pattern neurons,
$$\pi_i^{(\ell)} = 1-\left(1-\bar{z}^{(\ell)}\right)^{i-1}-(i-1) \bar{z}^{(\ell)} \left(1-\bar{z}^{(\ell)}\right)^{i-2}.$$
Now, letting $\pi^{(\ell)}(t)$ denote the average probability of sending erroneous nodes by super constraint nodes in plane $\ell$ and in iteration $t$, we will have
\begin{eqnarray*}
\pi^{(\ell)}(t) &=& \mb{E} \{\pi_i^{(\ell)}\}\\ &=& \sum_{i} \rho_{i} \pi_i^{(\ell)} \\
&=& 1-\rho(1-\bar{z}^{(\ell)}(t))-\bar{z}^{(\ell)}(t)\rho'(1-\bar{z}^{(\ell)}(t)),
\end{eqnarray*}
where $\rho(z) = \sum_i \rho_i z^{i-1}$ is the super constraint node degree distribution polynomial and $\rho'(z) = d\rho(z)/dz$. 

To simplify notations, let us define the function $g(z)=1-\rho(1-z)-z\rho'(1-z)$ such that
$$\pi^{(\ell)}(t) = g(\bar{z}^{(\ell)}(t)).$$
Now consider a given pattern neuron with degree $j$ in plane $\ell$. Let $z_j^{(\ell)}(t+1)$ denote the probability of sending an erroneous message by this node in iteration $t+1$. Then, $z_j^{(\ell)}(t+1)$ is equal to the probability of this node being noisy in the first place ($p_e$) and having all its super constraint nodes sending erroneous messages in iteration $t$, the average probability of which is $$\bar{\pi}^{(\ell)}(t)=\frac{1}{2\Omega + 1}\sum_{i=-\Omega}^{\Omega}\pi^{(\ell-i)}(t).$$ Now, since $z^{(\ell)}(t+1) = \mb{E} \{z_j^{(\ell)}(t+1)\}$, we get
\beqa
z^{(\ell)}(t+1) &=& p_e \sum_j \lambda_j \left( \bar{\pi}^{(\ell)} \right)^{j} \nonumber \\
&=&p_e \lambda(\bar{\pi}^{(\ell)}) \nonumber \\
&=& p_e \lambda (\frac{1}{2\Omega + 1}\sum_{i=-\Omega}^{\Omega}g(\bar{z}^{(\ell-i)}(t))) \nonumber
\eeqa

Again to simplify the notation, let us define the function $f(z;p_e) = p_e \lambda(z)$. This way, we will have the recursion as:
$$z^{(\ell)}(t+1) = f\left(\frac{1}{2\Omega + 1}\sum_{i=-\Omega}^{\Omega}g(\bar{z}^{(\ell-i)}(t));p_e\right).$$
\end{proof}

The decoding will be successful if $z^{(\ell)}(t+1) < z^{(\ell)}(t),\ \forall \ell$. As a result, we look for the maximum $p_e$ such that $$f\left(\frac{1}{2\Omega + 1}\sum_{i=-\Omega}^{\Omega}g(\bar{z}^{(\ell-i)}(t));p_e\right) < z^{(\ell)} ~\text{for}~ z^{(\ell)} \in [0,p_e].$$ Let $p_e^\dagger$ and $p_e^*$ be the maximum $p_e$'s that admit successful decoding for the uncoupled and coupled systems, respectively. To estimate these thresholds, we follow the approach recently proposed in \cite{pfister_istc} and define a potential function to track the evolution of Eq.~(\ref{z_ell}). Let $\textbf{z}=\{z^{(1)},\dots,z^{(L)}\}$ denote the vector of average probabilities of error for pattern neurons in each plane. Furthermore, let $\textbf{f}(\textbf{z};p_e): \mb{R}^L \rightarrow \mb{R}^L$ and $\textbf{g}(\textbf{z}): \mb{R}^L \rightarrow \mb{R}^L$ be two component-wise vector functions such that $[\textbf{f}(\textbf{z};p_e)]_i = f(z_i;p_e)$ and $[\textbf{g}(\textbf{z})]_i = g(z_i)$, where $f(z_i;p_e)$ and $g(z_i)$ are defined in Lemma~\ref{lem:recursive}.
Using these definitions, we can rewrite Eq.~(\ref{z_ell}) in the vector form as \cite{pfister_istc}:
\beq\label{z_vector}
\textbf{z}(t+1) = A^\top\textbf{f}(A\textbf{g}(\textbf{z}(t));p_e)
\eeq
where $A$ is the \emph{coupling matrix} defined as\footnote{Matrix $A$ corresponds to the unconstrained system. A similar matrix can be defined for the constrained case.}:
\begin{figure}[h!]
\begin{center}
\begin{tikzpicture}[scale=.5]
\node at (-2.9,5.1)[yshift=0,xshift=0,rectangle,minimum size=22mm,below delimiter=\{] (v3) {};
\node at (-2.9,3.05)[rectangle] (v1){$\Omega$};
\node at (-7,0)(v0) {$A=\frac{1}{2\Omega+1}$};
\node at (-2.3,-.3)[rectangle,minimum size=27mm,left delimiter={[}] (v1) {};
\node at (2.9,-.3)[rectangle,minimum size=27mm,right delimiter={]}] (v2) {};
\node at (-4.5,1.8)[rectangle] (v1){$1$};
\node at (-4.5,1.8)[xshift =5mm,yshift =0mm,rectangle] (v1){$1$};
\node at (-4.5,1.8)[xshift =12mm,yshift =0mm,rectangle] (v1){$\dots$};
\node at (-4.5,1.8)[xshift =17mm,yshift =0mm,rectangle] (v1){$1$};
\node at (-4.5,1.8)[xshift =22mm,yshift =0mm,rectangle] (v1){$0$};
\node at (-4.5,1.8)[xshift =27mm,yshift =0mm,rectangle] (v1){$0$};
\node at (-4.5,1.8)[xshift =32mm,yshift =0mm,rectangle] (v1){$0$};
\node at (-4.5,1.8)[xshift =39mm,yshift =0mm,rectangle] (v1){$\dots$};
\node at (-4.5,1.8)[xshift =44mm,yshift =0mm,rectangle] (v1){$0$};
\node at (-4.5,1.8)[xshift =49mm,yshift =0mm,rectangle] (v1){$0$};

\node at (-4.5,1.8)[xshift =0mm,yshift =-5mm,rectangle] (v1){$1$};
\node at (-4.5,1.8)[xshift =5mm,yshift =-5mm,rectangle] (v1){$1$};
\node at (-4.5,1.8)[xshift =12mm,yshift =-5mm,rectangle] (v1){$\dots$};
\node at (-4.5,1.8)[xshift =17mm,yshift =-5mm,rectangle] (v1){$1$};
\node at (-4.5,1.8)[xshift =22mm,yshift =-5mm,rectangle] (v1){$1$};
\node at (-4.5,1.8)[xshift =27mm,yshift =-5mm,rectangle] (v1){$0$};
\node at (-4.5,1.8)[xshift =32mm,yshift =-5mm,rectangle] (v1){$0$};
\node at (-4.5,1.8)[xshift =39mm,yshift =-5mm,rectangle] (v1){$\dots$};
\node at (-4.5,1.8)[xshift =44mm,yshift =-5mm,rectangle] (v1){$0$};
\node at (-4.5,1.8)[xshift =49mm,yshift =-5mm,rectangle] (v1){$0$};

\node at (-4.5,1.8)[xshift =0mm,yshift =-9mm,rectangle] (v1){$\vdots$};

\node at (-4.5,1.8)[xshift =0mm,yshift =-16mm,rectangle] (v1){$0$};
\node at (-4.5,1.8)[xshift =5mm,yshift =-16mm,rectangle] (v1){$0$};
\node at (-4.5,1.8)[xshift =12mm,yshift =-16mm,rectangle] (v1){$\dots$};
\node at (-4.5,1.8)[xshift =17mm,yshift =-16mm,rectangle] (v1){$0$};
\node at (-4.5,1.8)[xshift =22mm,yshift =-16mm,rectangle] (v1){$0$};
\node at (-4.5,1.8)[xshift =27mm,yshift =-16mm,rectangle] (v1){$1$};
\node at (-4.5,1.8)[xshift =32mm,yshift =-16mm,rectangle] (v1){$1$};
\node at (-4.5,1.8)[xshift =39mm,yshift =-16mm,rectangle] (v1){$\dots$};
\node at (-4.5,1.8)[xshift =44mm,yshift =-16mm,rectangle] (v1){$1$};
\node at (-4.5,1.8)[xshift =49mm,yshift =-16mm,rectangle] (v1){$1$};

\node at (-4.5,1.8)[xshift =0mm,yshift =-21mm,rectangle] (v1){$0$};
\node at (-4.5,1.8)[xshift =5mm,yshift =-21mm,rectangle] (v1){$0$};
\node at (-4.5,1.8)[xshift =12mm,yshift =-21mm,rectangle] (v1){$\dots$};
\node at (-4.5,1.8)[xshift =17mm,yshift =-21mm,rectangle] (v1){$0$};
\node at (-4.5,1.8)[xshift =22mm,yshift =-21mm,rectangle] (v1){$0$};
\node at (-4.5,1.8)[xshift =27mm,yshift =-21mm,rectangle] (v1){$0$};
\node at (-4.5,1.8)[xshift =32mm,yshift =-21mm,rectangle] (v1){$1$};
\node at (-4.5,1.8)[xshift =39mm,yshift =-21mm,rectangle] (v1){$\dots$};
\node at (-4.5,1.8)[xshift =44mm,yshift =-21mm,rectangle] (v1){$1$};
\node at (-4.5,1.8)[xshift =49mm,yshift =-21mm,rectangle] (v1){$1$};
\end{tikzpicture}
\end{center}
\end{figure}

At this point, the potential function of the unconstrained coupled system could be defined as \cite{pfister_istc}:
\beqa\label{eq:potential}
U(\textbf{z};p_e) &=& \int_{C} \textbf{g}'(\textbf{u})(\textbf{u}-A^\top \textbf{f}(A\textbf{g}(\textbf{u})).d\textbf{u}\nonumber \\
&=& \textbf{g}(\textbf{z})^\top\textbf{z}-G(\textbf{z})-F(A\textbf{g}(\textbf{z});p_e)
\eeqa
where $\textbf{g}'(\textbf{z}) = \hbox{diag}([g'(u_i)])$, $G(\textbf{z}) = \int_{C} \textbf{g}(\textbf{u})\cdot d\textbf{u}$ and $F(\textbf{z}) = \int_{C} \textbf{f}(\textbf{u})\cdot d\textbf{u}$. 

A similar quantity can be defined for the uncoupled (scalar) system as $U_s(z;p_e)=zg(z)-G(z)-F(g(z);p_e)$ \cite{pfister_istc}, where $z$ is the average probability of error in pattern neurons. The scalar potential function is defined in the way that $U_s'(z;p_e) > 0$ for $p_e \leq p_e^\dagger$. In other words, it ensures that $z(t+1) = f(g(z(t);p_e) < z(t)$ (successful decoding) for $p_e \leq p_e^\dagger$.

Furthermore, let us define $p_e^* = \sup\{p_e|\min(U_s(z;p_e) \geq 0\}$. Thus, in order to find $p_e^*$, it is sufficient to find the maximum $p_e$ such that $\min\{U_s(z;p_e)\} > 0$ \cite{pfister_istc}. We will show that the constrained coupled system achieves successful error correction for $p_e < p_e^*$. Intuitively, we expect to have $p_e^\dagger \leq p_e^*$ (side information only helps), and as a result a better error correction performance for the constrained system. Theorem~\ref{theorem:constrained_coupled} and our experimental result will confirm this intuition later in the paper.

Let $\Delta E(p_e) = \min_z U_s(z;p_e)$ be the \emph{energy gap} of the uncoupled system for $p_e \in (p_e^\dagger,1]$ \cite{pfister_istc}. The next theorem borrows the results of \cite{pfister_istc} and \cite{kudekar} to show that the constrained coupled system achieves successful error correction for $p_e < p_e^*$.
\begin{theorem}\label{theorem:constrained_coupled}
In the constrained system, when $p_e < p_e^*$ the potential function decreases in each iteration. Furthermore, if $\Omega > \frac{\Vert U''(\textbf{z};p_e)\Vert_{\infty}}{ \Delta E(p_e)}$, the only fixed point of Eq.~(\ref{z_vector}) is $0$.
\end{theorem}
\begin{proof}
The proof of the theorem relies on results from \cite{kudekar} to show that the entries in the vector $\textbf{z}(t) = [z^{(1)}(t),\dots,z^{(L)}(t)]$ are non-decreasing, i.e.,  $$z^{(1)}(t)\leq z^{(2)}(t)\leq \dots \leq z^{(L)}(t).$$ This can be shown using induction and the fact that the functions $\textbf{f}(\cdot,p_e)$ and $\textbf{g}(\cdot)$ are non-decreasing (see the proof of Lemma 22 in \cite{kudekar} for more details).

Then, one can apply the result of Lemma 3 in \cite{pfister_istc} to show that the potential function of the constrained coupled system decreases in each iteration. Finally, when $$\Omega > \Vert U''(\textbf{z};p_e)\Vert_{\infty} / \Delta E(p_e)$$ one could apply Theorem 1 of \cite{pfister_istc} to show the convergence of the probability of errors to zero.
\end{proof}
Note that Theorem~\ref{theorem:constrained_coupled} provides a sufficient condition (on $\Omega$) for the coupled system to ensure it achieves successful error correction for every $p_e$ upto $p_e = p_e^*$. However, the condition provided by Theorem~\ref{theorem:constrained_coupled} usually requires $\Omega$ to be \emph{too large}, i.e. $\Omega$ is required to be as large as $1000$ to $10000$, depending on the degree distributions. Nevertheless, in the next section we show that the analysis is still quite accurate for moderate values of $\Omega$, i.e. $\Omega \simeq 2,3$, meaning that a system with a small coupling parameters could still achieve very good error correction in practice. 

\section{Pattern retrieval capacity}
The following theorem shows that the number of patterns that can be memorized by the proposed scheme is exponential in $n$, the pattern size. 
\begin{theorem}\label{theorem_exponential_solution}
Let $\mc{X}$ be the $\mc{C} \times n$ dataset matrix, formed by $\mc{C}$ vectors of length $n$ with entries from the set $\mc{S}$. Let also $k = rn$ for some $0<r<1$. Then, there exists a set of patterns for which $\mc{C} = a^{rn}$, with $a > 1$, and $\hbox{\text{rank}}(\mc{X}) = k <n$.
\end{theorem}
\begin{proof}
The proof is based on construction: we construct a data set $\mc{X}$ with the required properties such that it can be memorized by the proposed neural network. To simplify the notations, we assume all the clusters have the same number of pattern and constraint neurons, denote by $\widetilde{n}_c$ and $\widetilde{m}_c$. In other words, $n_{\ell,d} = \widetilde{n}_c$ and $m_{\ell,d} = \widetilde{m}_c $ for all $\ell=\{1,\dots,L\}$ and $d=\{1,\dots,D\}$.

We start by considering a matrix $G \in \mb{R}^{k \times n}$, with non-negative integer-valued entries between $0$ and $\gamma-1$ for some $\gamma \geq 2$. We also assume $k = rn$, with $0 < r < 1$.

To construct the database, we first divide the columns of $G$ into $L$ sets, each corresponding to the neurons in one plain. Furthermore, in order to ensure that all the sub-patterns within each cluster form a subspace with dimension less than $\widetilde{n}_c$, we propose the following structure for the generator matrix $G$. This structure ensures that the rank of any sub-matrix of $G$ composed of $\widetilde{n}_c$ columns is less than $\widetilde{n}_c$. In the matrices below, the hatched blocks represent parts of the matrix with \emph{some} non-zero entries. To simplify visualization, let us first define the sub-matrix $\hat{G}$ as the building blocks of $G$:
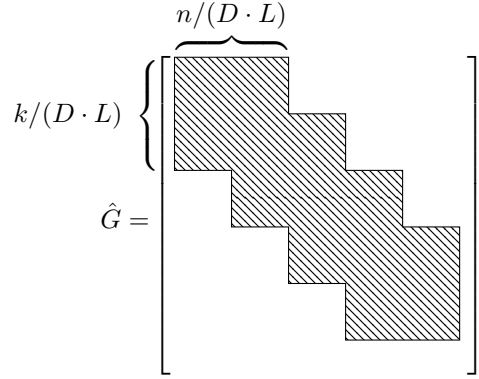
\vspace{-1cm}
\begin{center}
\begin{tikzpicture}[scale=1,xcenter around={0,0}{.335\textwidth,8}]
\node at (0.35,-2.1) (v0) {$\hat{G}=$};
\node at (0,0)[yshift=-2.1cm,xshift=3.1cm,rectangle,minimum size=42mm,left delimiter={[}] (v1) {};
\node at (0,0)[yshift=-2.1cm,xshift=2.7cm,rectangle,minimum size=42mm,right delimiter={]}] (v2) {};
\draw[pattern=north west lines][xshift=1cm] (0,0)--(1.5,0)--(1.5,-0.75)--(2.25,-0.75)--(2.25,-1.5)--(3,-1.5)--(3,-2.25)--(3.75,-2.25)--(3.75,-3.75)--(2.25,-3.75)--(2.25,-3)--(1.5,-3)--(1.5,-2.25)--(0.75,-2.25)--(0.75,-1.5)--(0,-1.5)--cycle;
\node at (1.5,1.3)[yshift=-.15cm,xshift=0.25cm,rectangle,minimum size=15mm,below delimiter=\{] (v3) {};
\node at (1.5,1.15)[yshift=-.6cm,xshift=0.25cm] (v8) {$n/(D\cdot L)$};
\node at (-1.9,-.65)[yshift=-3.5,xshift=1.5cm,rectangle,minimum size=15mm,right delimiter=\{] (v4) {$k/(D\cdot L)$};
\end{tikzpicture}
\end{center}

Then, $G$ is structured as

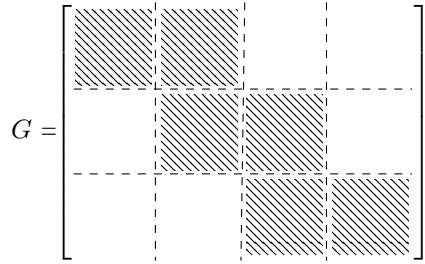
\begin{center}

\begin{tikzpicture}[scale=0.25,xcenter around={0,0}{.095\textwidth,8}]

\node at (-11,2)(v0) {$G$ = };
\node at (-2,1.8)[rectangle,minimum size=34mm,left delimiter={[}] (v1) {};
\node at (2,1.8)[rectangle,minimum size=34mm,right delimiter={]}] (v2) {};
\node at (-6.8,6.3)[rectangle,minimum size=10mm,pattern=north west lines] (v2) {};
\draw[-,dashed](-4.6,8.7)--(-4.6,-5);
\node at (-2.3,6.3)[rectangle,minimum size=10mm,pattern=north west lines] (v2) {};
\draw[-,dashed](0.1,8.7)--(-.1,-5);
\node at (-2.3,1.8)[rectangle,minimum size=10mm,pattern=north west lines] (v2) {};
\draw[-,dashed](4.4,8.7)--(4.4,-5);
\node at (2.2,1.8)[rectangle,minimum size=10mm,pattern=north west lines] (v2) {};
\node at (2.2,-2.7)[rectangle,minimum size=10mm,pattern=north west lines] (v2) {};
\node at (6.7,-2.7)[rectangle,minimum size=10mm,pattern=north west lines] (v2) {};
\draw[-,dashed](-8.9,4.1)--(8.9,4.1);
\draw[-,dashed](-8.9,-.4)--(8.9,-.4);
\end{tikzpicture}
\end{center}
where each hatched block represents a random realization of $\hat{G}$.

Now consider a random vector $u \in \mb{R}^k$ with integer-valued-entries between $0$ and $\upsilon-1$, where $\upsilon \geq 2$. We construct the dataset by assigning the pattern $\textbf{x} \in \mc{X}$ to be $\textbf{x} = \textbf{u} \cdot G$, \emph{if} all the entries of $\textbf{x}$ are between $0$ and $S-1$. Obviously, since both $\textbf{u}$ and $G$ have only non-negative entries, all entries in $\textbf{x}$ are non-negative. Therefore, it is the $S-1$ upper bound that we have to worry about. 

Let $\varrho_j$ denote the $j^{th}$ column of $G$. Then the $j^{th}$ entry in $\textbf{x}$ is equal to $x_j = \textbf{u} \cdot \varrho_j$. Suppose $\varrho_j$ has $d_j$ non-zero elements. Then, we have: 
\bed
x_j = u \cdot \varrho_j \leq d_j (\gamma-1) (\upsilon-1)
\eed

Therefore, letting $d^* = \max_{j} d_j$, we could choose $\gamma$, $\upsilon$ and $d^*$ such that
\beq\label{capaci_inequal}
S-1 \geq d^* (\gamma-1) (\upsilon-1)
\eeq
to ensure all entries of $x$ are less than $S$. 

As a result, since there are $\upsilon^k$ vectors $u$ with integer entries between $0$ and $\upsilon -1$, we will have $\upsilon^k = \upsilon^{rn}$ patterns forming $\mc{X}$. Which means $\mc{C} = \upsilon^{rn}$, which would be an exponential number in $n$ if $\upsilon \geq 2$. 
\end{proof}

\section{Simulations}\label{sec:simulation} 
In this paper, we are mainly interested in the performance of the recall phase and demonstrate a way, by the means of spatial coupling, to improve upon the previous art. To this end, we assume that the learning phase is done (by using our proposed algorithm in \cite{KSS_ICML2013}) and  we have the weighted connectivity graphs available. For the ease of presentation, we can simply  produce these matrices by generating sparse random bipartite graphs and assign random weights to the connections. Given the weight matrices and the fact that they are orthogonal to the sub-patterns, we can assume w.l.o.g that in the recall phase we are interested in recalling the all-zero pattern from its noisy version.

We treat the patterns in the database as $2D$ images of size $64\times 64$.  More precisely, we have generated a random network with $29$ planes and $29$ clusters within each plane (i.e., $L=D=29$). Each local cluster is composed of $8\times8$ neurons and each pattern neuron (pixel) is connected to $2$ consecutive planes and $2$ clusters within each plane (except at the boundaries). This is achieved by  moving the $8\times 8$ rectangular window over the $2D$ pattern horizontally and vertically. The degree distribution of this setting is $\lambda = \{0.0011,0.0032,0.0043,0.0722,0,0.0054,0,0.0841,0.0032,0,0,0.098,0,0,0,0.7284 \}$, $\rho_{64} = 1$ and $\rho_j = 0$ for $1 \leq j \leq 63$.

We investigated the performance of the recall phase by randomly generating a $2D$ noise pattern in which each entry is set to $\pm 1$ with probability $p_e/2$ and $0$ with probability $1-p_e$. We then apply Algorithm~\ref{algo:coupled} with $t_{\max} = 10$ to eliminate the noise. Once finished, we declare failure if the output of the algorithm, $\hat{\textbf{x}}$, is not equal to  the  pattern $\textbf{x}$ (assumed to be the all-zero vector). 

Figure~\ref{fig:PER_coupled} illustrates the final error rate of the proposed algorithm, for the constrained and unconstrained system. For the constrained system, we fixed the state of a patch of neurons of size $3 \times 3$ at the four corners of the $2D$ pattern. The results are also compared to the similar algorithms in \cite{KSS} and \cite{KSS_ICML2013} (uncoupled systems). In \cite{KSS} (the dashed-dotted curve), there are no clustering while in \cite{KSS_ICML2013} the network is divided into $50$ overlapping clusters all lying on a single plane (the dotted curve). Although clustering improves the performance, it is still inferior than the coupled system with some side information (the solid curve). Even though the same recipe (i.e., Alg.~\ref{algo:correction}) is used in all approaches, the differences in the architectures has a profound effect on the performance. One also notes the sheer positive effect of network size on the performance (the dotted vs. dashed curves).

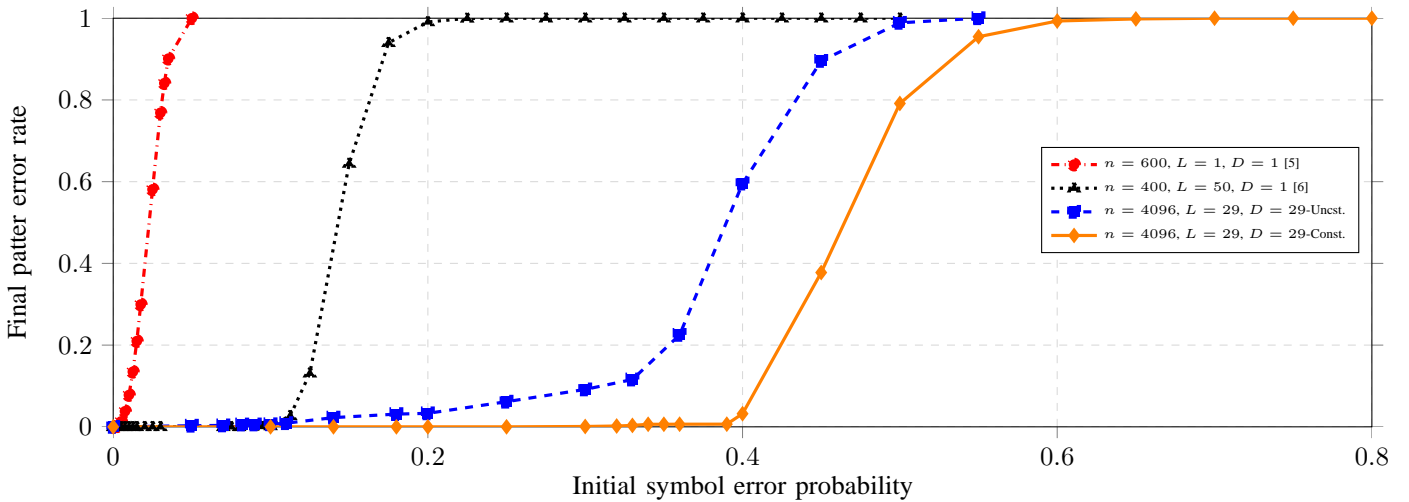
\begin{figure}[t]
\centering
\begin{tikzpicture}
    \begin{axis}[
        width=\textwidth, height=7cm,     
        grid = major,
        grid style={dashed, gray!30},
        xmin=0,     
        xmax=.8,    
        ymin=0,     
        ymax=1.0,   
        /pgfplots/xtick={0,.2,...,.8}, 
        axis background/.style={fill=white},
        ylabel=Final patter error rate,
        xlabel=Initial symbol error probability,
        tick align=outside,
	legend style={
		cells={anchor=west},
		font = \tiny,
		at={(.99,.685)},
		}
	]
 
      \addplot [color=red,dashdotted,very thick,mark=*] table {./PER_Spatially_Coupled_Fig_Data_N_600_L_1_D_1_ITW11.dat};
      \addplot [color=black,dotted,very thick,mark=triangle*] table {./PER_Spatially_Coupled_Fig_Data_N_400_L_50_D_1_ICML13.dat};
      \addplot [color=blue,dashed,very thick,very thick,mark=square*] table {./PER_Spatially_Coupled_Fig_Data_N_4096_L_29_D_29_UnConst.dat};
      \addplot [color=orange,very thick,very thick,mark=diamond*] table {./PER_Spatially_Coupled_Fig_Data_N_4096_L_29_D_29_Const.dat};

      \legend{$n=600$, $L=1$, $D=1$ \cite{KSS}\\ $n=400$, $L=50$, $D=1$ \cite{KSS_ICML2013}\\ $n=4096$, $L=29$, $D=29$-Uncst. \\ $n=4096$, $L=29$, $D=29$-Const. \\}		            
   \end{axis} 
\end{tikzpicture}
\caption{The final pattern error probability for the constrained and unconstrained coupled neural systems.}
\label{fig:PER_coupled}
\end{figure}

Table \ref{table:thresholds} shows the thresholds $p_e^\dagger$ and $p_e^*$ for different values of $e$. From Figure~\ref{fig:PER_coupled} we notice that $ p_e^* \simeq 0.39$ and $p_e^\dagger \simeq .1$ which is close to the thresholds for $e=2$ in Table~\ref{table:thresholds}. Note that according to Theorem~\ref{theorem:constrained_coupled}, a sufficient condition for these thresholds to be exact is for $\Omega$ to be very large. However, the comparison between Table~\ref{table:thresholds} and Figure~\ref{fig:PER_coupled} suggest, one could obtain rather exact results even with $\Omega$ being rather small.
\begin{table}[h]
\centering
\begin{tabular}{|c|c|c|}
\hline
 &$p_e^\dagger$ & $p_e^*$ \\ \hline
$e = 1$ &$0.078$ & $0.114$ \\ \hline
$e = 2$ &$0.197$ & $0.394$ \\ \hline
\end{tabular}
\caption{The thresholds for the uncoupled ($p_e^\dagger$) and coupled ($p_e^*$) systems.}
\label{table:thresholds}
\end{table}

Figure~\ref{fig:potential_coupled} illustrates how the potential function for uncoupled systems behaves as a function of $z$ and for various values of $p_e$. Note that for $p_e \simeq p_e^*$, the minimum value of potential reaches zero, i.e. $\Delta_E(p_e^*) = 0$, and for $p_e > p_e^*$ the potential becomes negative for large values of $z$.

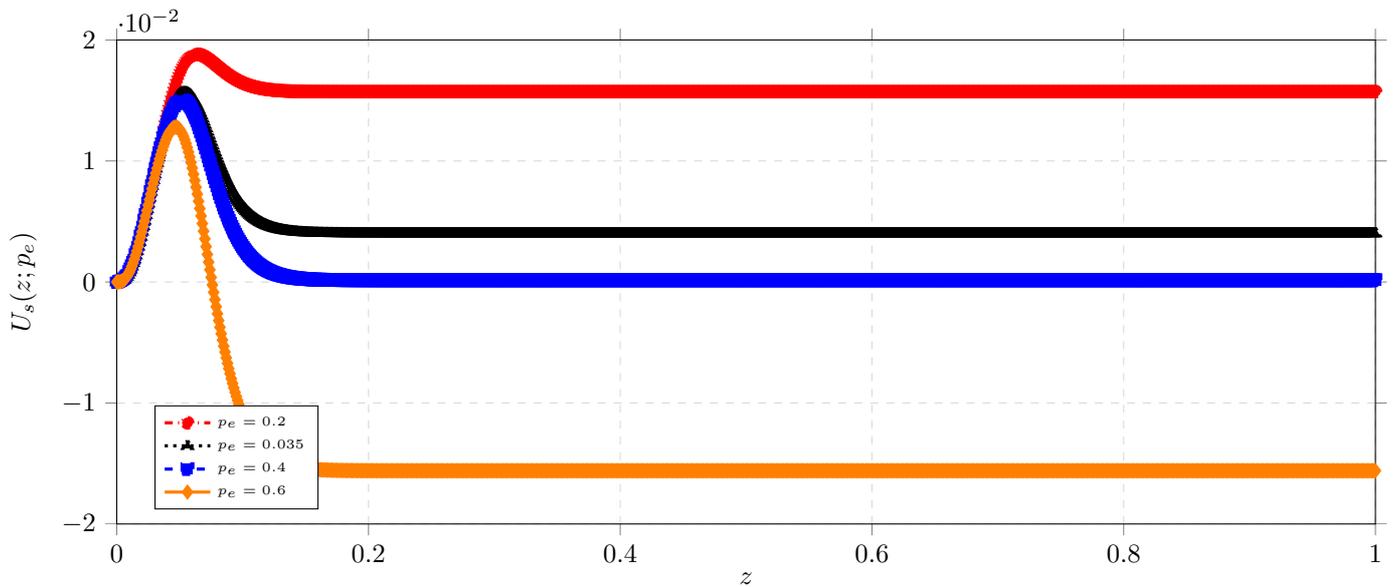
\begin{figure}
\centering
\begin{tikzpicture}
    \begin{axis}[
        width=\textwidth, height=8cm,     
        grid = major,
        grid style={dashed, gray!30},
        xmin=0,     
        xmax=1,    
        ymin=-.02,     
        ymax=.02,   
        /pgfplots/xtick={0,.2,...,1}, 
        axis background/.style={fill=white},
        ylabel=$U_s(z;p_e)$,
        xlabel=$z$,
        tick align=outside,
	legend style={
		cells={anchor=west},
		legend pos=south west,
		font = \tiny,
		}
	]
 
      \addplot [color=red,dashdotted,very thick,mark=*] table {./Potential_Coupled_Fig_Data_N_64_L_29_p_e_02.dat};
      \addplot [color=black,dotted,very thick,mark=triangle*] table {./Potential_Coupled_Fig_Data_N_64_L_29_p_e_035.dat};
      \addplot [color=blue,dashed,very thick,very thick,mark=square*] table {./Potential_Coupled_Fig_Data_N_64_L_29_p_e_04.dat};
      \addplot [color=orange,very thick,very thick,mark=diamond*] table {./Potential_Coupled_Fig_Data_N_64_L_29_p_e_06.dat};

      \legend{$p_e = 0.2$\\ $p_e = 0.035$\\ $p_e = 0.4$\\ $p_e = 0.6$\\}		            
   \end{axis} 
\end{tikzpicture}
\caption{The scalar potential function $U_s$ as function of average pattern neurons error probability, $z$, and different initial symbol error probabilities, $p_e$.}
\label{fig:potential_coupled}
\end{figure}

\section{Conclusions}\label{sec:conclusions}
In this paper, we proposed a novel architecture for neural associative memories. The proposed model comprises a set of neural planes with sparsely connected overlapping clusters. Furthermore, planes are sparsely connected together as well.

Given the similarity of the suggested framework to spatially-coupled codes, we employed recent developments in analyzing these codes to  investigate the performance of our proposed neural algorithm. We also presented numerical simulations that lend additional support to the theoretical analysis. 
We derived two thresholds on the maximum initial bit error probability that can be corrected by the proposed algorithm with probability close to $1$. Using simulations, we confirmed that there is a good match between the thresholds derived theoretically and those obtained in practice.

Given that our main interest in this paper was the performance of the error correcting algorithm in the recall phase, we did not address the learning phase here. However, we are currently in the middle of applying the learning method in \cite{KSS_ICML2013} to a database of natural images to assess the performance of the recall algorithm in this real-world setup as well. 

\subsection*{Acknowledgments}
The authors would like to thank Prof. Henry D. Pfister, Mr. Vahid Aref, and Dr. Seyed Hamed Hassani for their helpful comments and discussions.

\begin{small}

\end{small}

%
%
%
%
%
%
%
%
%
%

\end{document}